\def\year{2017}\relax
\documentclass[letterpaper]{article}
\usepackage{aaai17}
\usepackage{times}
\usepackage{helvet}
\usepackage{courier}
\usepackage{subfig}

\usepackage[latin9]{inputenc}
\usepackage{amsmath,amsfonts,amsthm}
\usepackage{nicefrac}
\usepackage{mathtools}
\usepackage{algorithm}
\usepackage{algorithmic}
\usepackage{tabularx}
\usepackage{multirow}

\usepackage{color}
\usepackage{array}
\usepackage{rotating}

\DeclarePairedDelimiter{\brac}{[}{]}
\DeclarePairedDelimiter{\braces}{\{}{\}}
\DeclarePairedDelimiter{\paren}{\lparen}{\rparen}

\newcommand{\AlgName}{AlphaMax}

\newcommand{\mc}[1]{\mathcal{#1}}

\newcommand{\cone}{f_1}
\newcommand{\czero}{f_0}
\newcommand{\mixu}{f}
\newcommand{\mixl}{g}
\newcommand{\mpu}{\alpha}
\newcommand{\mpl}{\beta}

\newcommand{\cls}{h}

\newcommand{\tpr}{\gamma}
\newcommand{\fpr}{\eta}
\newcommand{\pr}{\rho}
\newcommand{\tprhat}{\hat{\gamma}}

\newcommand{\prhat}{\hat{\rho}}
\newcommand{\fprPU}{\eta^{pu}}
\newcommand{\prPU}{\rho^{pu}}

\newcommand{\fprLU}{\eta^{pu}}
\newcommand{\tprLU}{\gamma^{pu}}
\newcommand{\fprPUhat}{\hat{\fpr}^{pu}}
\newcommand{\fprLUhat}{\hat{\fpr}^{pu}}
\newcommand{\tprLUhat}{\hat{\tpr}^{pu}}
\newcommand{\prLUhat}{\hat{\rho}^{pu}}
\newcommand{\prLU}{\rho^{pu}}
\newcommand{\rocPU}{\text{ROC$^{pu}$}}
\newcommand{\aucPU}{\text{AUC$^{pu}$}}

\newcommand{\aucLU}{\text{AUC$^{pu}$}}
\newcommand{\aucREC}{\text{AUC$^{est}$}}
\newcommand{\aucPN}{\text{AUC}}

\newcommand{\xset}{\mc{X}}
\newcommand{\yset}{\mc{Y}}
\newcommand{\Xu}{X}
\newcommand{\Xone}{X_1}
\newcommand{\xu}{x}

\newcommand{\Xl}{\Xone}
\newcommand{\xl}{x}

\newcommand{\EE}{\mathbb{E}}

\ifx\theorem\undefined
\newtheorem{theorem}{Theorem}
\fi

\title{Recovering True Classifier Performance in Positive-Unlabeled Learning}

\author{
  Shantanu Jain, Martha White, Predrag Radivojac \\
  Department of Computer Science\\
       Indiana University, Bloomington, Indiana, USA\\
     \{shajain, martha, predrag\}@indiana.edu \\
  }

\begin{document}

\maketitle

\begin{abstract}
A common approach in positive-unlabeled learning is to train a classification model between labeled and unlabeled data. This strategy is in fact known to give an optimal classifier under mild conditions; however, it results in biased empirical estimates of the classifier performance. In this work, we show that the typically used performance measures such as the receiver operating characteristic curve, or the precision-recall curve obtained on such data can be corrected with the knowledge of class priors; i.e., the proportions of the positive and negative examples in the unlabeled data. We extend the results to a noisy setting where some of the examples labeled positive are in fact negative and show that the correction also requires the knowledge of the proportion of noisy examples in the labeled positives. Using state-of-the-art algorithms to estimate the positive class prior and the proportion of noise, we experimentally evaluate two correction approaches and demonstrate their efficacy on real-life data.
\end{abstract}

\section{Introduction}
Performance estimation in binary classification is tightly related to the nature of the classification task. As a result, different performance measures may be directly optimized during training. When (mis)classification costs are available, the classifier is ideally trained and evaluated in a cost-sensitive mode to minimize the expected cost \cite{Whalen1971,Elkan2001}. More often, however, classification costs are unknown and the overall performance is assessed by averaging the performance over a range of classification modes. The most extensively studied and widely used performance evaluation in binary classification involves estimating the Receiver Operating Characteristic (ROC) curve that plots the true positive rate of a classifier as a function of its false positive rate \cite{Fawcett2006}. The ROC curve provides insight into trade-offs between the classifier's accuracies on positive versus negative examples over a range of decision thresholds. Furthermore, the area under the ROC curve (AUC) has a meaningful probabilistic interpretation that correlates with the ability of the classifier to separate classes and is often used to rank classifiers \cite{Hanley1982}. Another important performance criterion generally used in information retrieval relies on the precision-recall (pr-rc) curve, a plot of precision as a function of recall. The precision-recall evaluation, including summary statistics derived from the pr-rc curve, may be preferred to ROC curves when classes are heavily skewed \cite{Davis2006}. 

Although model learning and performance evaluation in a supervised setting are well understood (Hastie et al. \citeyear{Hastie2001}), the availability of unlabeled data gives additional options and also presents new challenges. A typical semi-supervised scenario involves the availability of positive, negative and (large quantities of) unlabeled data. Here, the unlabeled data can be used to improve training \cite{Blum1998} or unbias the labeled data \cite{Cortes2008}; e.g., to estimate class proportions that are necessary to calibrate the model and accurately estimate precision when class balances (but not class-conditional distributions) in labeled data are not representative (Saerens et al. \citeyear{Saerens2002}). This is often the case when it is more expensive or difficult to label examples of one class than the examples of the other. A special case of the semi-supervised setting arises when the examples of only one class are labeled. It includes open-world domains such as molecular biology where, for example, wet lab experiments determining a protein's activity are generally conclusive; however, the absence of evidence about a protein's function cannot be interpreted as the evidence of absence. This is because, even when the labeling is attempted, a functional assay may not lead to the desired activity for a number of experimental reasons. In other domains, such as social networks, only positive examples can be collected (such as `liking' a particular product) because, by design, the negative labeling is not allowed. The development of classification models in this setting is often referred to as positive-unlabeled learning (Denis et al. \citeyear{Denis2005}).

State-of-the-art techniques in positive-unlabeled learning tackle this problem by treating the unlabeled sample as negatives and training a classifier to distinguish between labeled (positive) and unlabeled examples. Following Elkan and Noto~\shortcite{Elkan2008}, we refer to the classifiers trained on a labeled sample from the true distribution of inputs, containing both positive and negative examples, as \emph{traditional classifiers}. Similarly, we refer to the classifiers trained on the labeled versus unlabeled data as \emph{non-traditional classifiers}. In theory, the true performance of both traditional and non-traditional classifiers can be evaluated on a labeled sample from the true distribution (traditional evaluation). However, this is infeasible for non-traditional learners because such a sample is not available in positive-unlabeled learning. As a result, the non-traditional classifiers are evaluated by using the unlabeled sample as substitute for labeled negatives (non-traditional evaluation). Surprisingly, for a variety of performance criteria, non-traditional classifiers achieve similar performance under traditional evaluation as optimal traditional classifiers (Blanchard et al.~\citeyear{Blanchard2010}; Menon et al.~\citeyear{Menon2015}). The intuition for these results comes from the fact that in many practical situations, the posterior distributions in traditional and non-traditional setting provide the same optimal ranking of data points on a given test sample \cite{Jain2016,jain2016estimating}. Furthermore, the widely-accepted evaluation approaches using ROC or pr-rc curves are insensitive to the variation of raw prediction scores unless they affect the ranking.

Though the efficacy of non-traditional classifiers has been thoroughly studied \cite{Peng2003,Elkan2008,Ward2009,Menon2015}, estimating their true performance has been much less explored. Such performance estimation often involves computing the fraction(s) of correctly and incorrectly classified examples from both classes; however, in absence of labeled negatives, the fractions computed under the non-traditional evaluation are incorrect, resulting in biased estimates. Figure \ref{fig:roc} illustrates the effect of this bias by showing the traditional and non-traditional ROC curves on a handmade data set. Because some of the unlabeled examples in the training set are in fact positive, the area under the ROC curve estimated when the unlabeled examples were considered negative (non-traditional setting) underestimates the true performance for positive versus negative classification (traditional setting).

This paper formalizes and evaluates performance estimation of a non-traditional classifier in the traditional setting when the only available training data are (possibly noisy) positive examples and unlabeled data. We show that the true (traditional) performance of such a classifier can be recovered with the knowledge of class priors and the fraction of mislabeled examples in the positive set. We derive formulas for converting the ROC and pr-rc curves from the non-traditional to the traditional setting. Using these recovery formulas, we present methods to estimate true classification performance. Our experiments provide evidence that the methods for the recovery of a classifier's performance are sound and effective.

\renewcommand\floatpagefraction{.9}
\renewcommand\dblfloatpagefraction{.9} % for two column documents
\renewcommand\topfraction{.9}
\renewcommand\dbltopfraction{.9} % for two column documents
\renewcommand\bottomfraction{.9}
\renewcommand\textfraction{.1}   
\setcounter{totalnumber}{50}
\setcounter{topnumber}{50}
\setcounter{bottomnumber}{50}
\begin{figure}[h]
%\vspace{-4.0cm}
\includegraphics[width = \columnwidth]{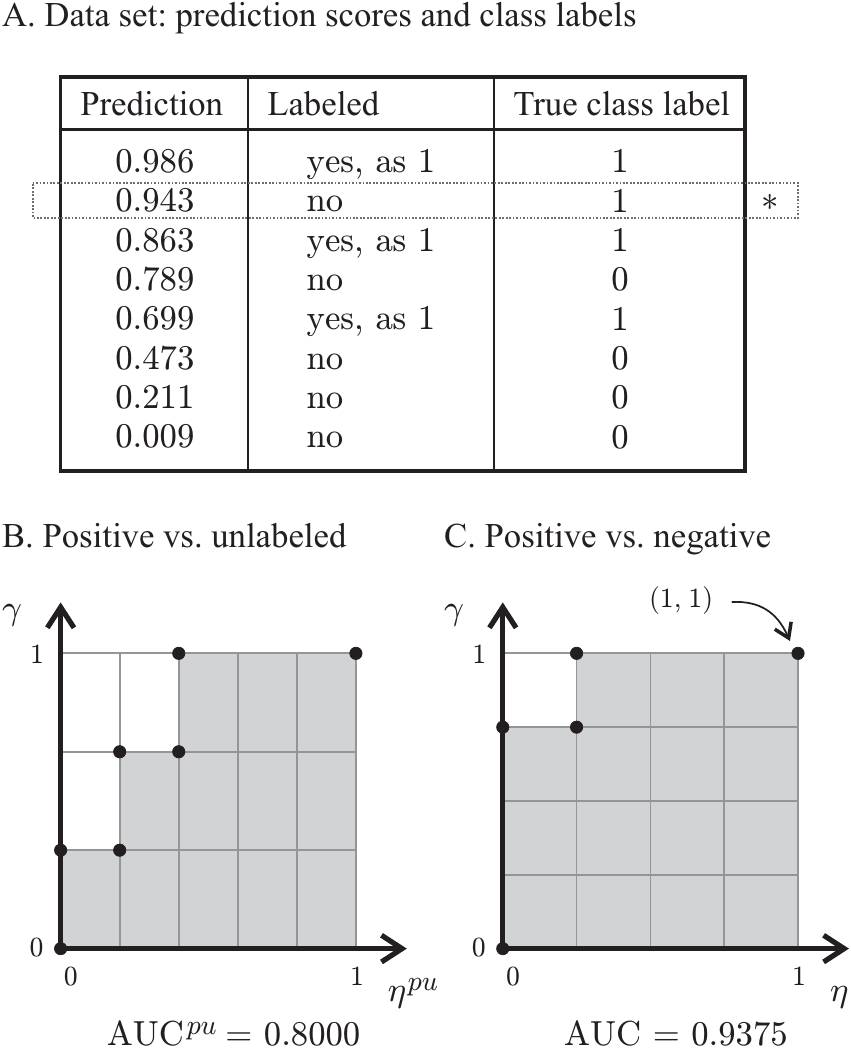}
\protect\caption{\small Illustration of the difference in classifier evaluation. (A) A data set with eight examples, three labeled positive and five unlabeled. One unlabeled example is positive (marked by an asterisk outside the table), whereas four are negative. A prediction score between zero and one is provided for each example. (B) The ROC plot ($\gamma$ = true positive rate; $\eta$ = false positive rate) when all unlabeled examples are considered negative. %The area under the ROC curve is calculated as $\textrm{AUC}^{\textrm{pu}}=\frac{1}{15} \times 12 = \frac{12}{15} = 0.8000$. 
(C) The true ROC plot where all examples are correctly labeled. %The area under the ROC curve is calculated as $\textrm{AUC}^{\textrm{pn}}=\frac{1}{16} \times 15 = \frac{15}{16} = 0.9375$. 
The areas under the ROC curves are calculated without interpolation as the total area of shaded boxes.}
\label{fig:roc}

\end{figure}

\section{Problem formulation}

Consider a binary classification problem from input $x\in\xset$ to output $y \in \yset=\{0,1\}$ in a positive-unlabeled setting. Let $\mixu$ be the true distribution over the input space $\xset$ from which the unlabeled sample is drawn and let $\cone$ and $\czero$ be the distributions of the positive and negative examples, respectively. It follows that $\mixu$ can be expressed as a two-component mixture containing $\cone$ and $\czero$ as
$$f(x)=\mpu \cone(x) + (1-\mpu)\czero(x),$$
for all $x\in \xset$ where $\mpu \in [0,1)$ is the mixing proportion (positive class prior) giving the proportion of positives in $\mixu$. 

Let now $\mixl$ be the distribution over $\xset$ from which the labeled sample is drawn. We similarly express $\mixl$ as a two-component mixture containing $\cone$ and $\czero$ as
$$g(x)=\mpl \cone(x) + (1-\mpl)\czero(x),$$
for all $x\in \xset$ where $\mpl \in (\mpu,1]$ gives the proportion of positives in labeled data. All labeled examples are labeled as positives; thus, when $\mpl = 1$ we say that the labeled data is clean. When $\mpl < 1$, the labeled data contains a fraction ($1-\mpl$) of negatives that are in this case mislabeled. We will refer to the latter scenario as the noisy positive setting.

Let $\Xl$ be the (positively) labeled sample drawn according to $\mixl(x)$ and $\Xu$ be the unlabeled sample drawn according to $\mixu(x)$. The learning objective is to train a classifier that discriminates between positive and negative data and estimate its performance. However, we can only train a non-traditional classifier $\cls: \xset \rightarrow \yset$ between labeled and unlabeled data and estimate its performance by considering that all labeled data are positive and all unlabeled data are negative. We refer to the performance of $\cls(x)$ directly estimated from samples $\Xl$ and $X$ as $\textrm{perf}^{\, pu}$. Given a non-traditional classifier $\cls(x)$ and its performance $\textrm{perf}^{\, pu}$, the main goal of this work is to estimate (recover) its performance in the traditional setting; i.e., its performance as a discriminator between positive and negative data.

\section{Methods}

We consider a family of binary classifiers that map $\xset$ into $\yset$. To simplify the presentation, we can think of the entire family as generated from a single model that maps $\xset$ into $\mathbb{R}$, where each individual classifier corresponds to a decision threshold picked from $\mathbb{R}$. The classifier gives the positive class `1' when the model's output is above the threshold and the negative class `0' otherwise. 

The true positive rate (sensitivity, recall) of each classifier is defined as the probability of correctly predicting a positive example; the true negative rate (specificity) is defined as the probability of correctly predicting a negative example; the false positive rate is defined as 1~$-$~specificity, and the false negative rate is defined as 1~$-$~sensitivity. Finally, the precision is defined as the probability that a positive prediction is correct; conversely, the false discovery rate is defined as 1~$-$~precision (Hastie et al. \citeyear{Hastie2001}). Given a test set, each of the quantities above is estimated using relative frequencies. In this setup, each classifier corresponds to a single confusion matrix, whereas the entire family of classifiers corresponds to a particular ROC curve and a particular pr-rc curve \cite{Fawcett2006}. The two main performance criteria considered in this work are the area under the ROC curve and the area under the pr-rc curve.

\subsection{The case of clean positive data}

We first consider the setting of clean positive data, where
the labeled data does not incorrectly contain negatives ($\beta = 1$),
to provide intuition before moving to the more general noisy-positive setting. For a classifier $\cls: \xset \rightarrow \yset$, the true positive rate, $\tpr$, and false positive rate, $\fpr$, can be defined as
\begin{align*}
    \tpr&= \EE_{\cone} [\cls(x)] \\
    \fpr&= \EE_{\czero} [\cls(x)],
\end{align*}
where $\EE_f$ denotes expectation with respect to a distribution $f$. The goal is to estimate these values, despite the fact that we only have access to positive labels. 

The true positive rate can be estimated as the empirical mean of $\cls(x)$ over the positively labeled sample $\Xone$
\begin{equation*}
%\hat{\tpr}=\frac{1}{|\Xone|}\sum_{i=1}^{|\Xone|}\cls(\xone_i)
\hat{\tpr}=\frac{1}{|\Xone|}\sum_{x \in \Xone}\cls(x)
\end{equation*}
because $\Xone$ was sampled from $\cone$. The false positive rate, however, cannot be so simply estimated, because we do not have access to a sample from $\czero$. Further, this prevents the estimation of the ROC curve and the area under this curve (AUC). Typically, ROC curves and AUCs are reported based only on the performance of the non-traditional positive-unlabeled classifier, $\cls$, on discriminating between positives and unlabeled data. The ROC curve for the positive versus unlabeled classification, \rocPU, can be estimated by plotting $\hat{\tpr}$ against $\fprPUhat$ across different classifiers, where $\fprPUhat$, an estimate of $\fprPU=\EE_\mixu[\cls(x)]$, can be estimated using the unlabeled sample (which corresponds to the negative sample for the non-traditional positive-unlabeled classifier $\cls$):
\begin{equation*}
\fprPUhat=\frac{1}{|\Xu|}\sum_{x \in \Xu}\cls(\xu).
\end{equation*}
This curve, however, does not represent the true performance of $\cls$ for positive versus negative classification. Similar difficulties exist in estimating the precision 
$$\pr = \frac{\mpu \EE_{\cone}[\cls(x)]}{\EE_\mixu[\cls(x)]}$$
that requires the positive class prior $\alpha$, though recall, which is equal to $\tpr$, can be directly estimated.

Of key interest, therefore, is a correction approach that provides an estimate of the true performance. We provide just such a result in Theorem 1 below for the more general setting of noisy positives (see next Section). Using this theorem for $\beta = 1$, for example, we can express the false positive rate $\fpr$ in terms of the positive-unlabeled false positive rate, $\fprPU$ as\footnote{Iakoucheva et al.~\shortcite{Iakoucheva2004} also provide this result for uncorrupted positive data.}
\begin{equation*}
{\fpr}= \frac{\fprPU - {\mpu}{\tpr}}{1-{\mpu}},
\end{equation*}
and the AUC of the classifier on the positive-negative classification problem in terms of the AUC of the classifier on the positive-unlabeled classification problem\footnote{Menon et al.~\shortcite{Menon2015} provide an equivalent formula for the AUC. In Theorem 1, we give a full derivation from the probabilistic definition of the AUC and conversion formulas for other measures.}:
\begin{equation*}
{\aucPN} = \frac{{\aucPU} -\frac{\mpu}{2}}{1-\mpu}.
\end{equation*}
Therefore, given estimates of $\mpu, \tpr, \fprPU$ and $\aucPU$, we can obtain estimates of $\aucPN$ and the precision. % $\pr = \frac{\tpr}{\tpr + \fpr}$. 
In the next Section, we present this key result that enables this conversion and also shows that the estimated \aucPN\ is better than \aucPU.

\subsection{The case of (possibly) noisy positive data}

In this section we consider a more general case where the labeled sample of positives is allowed to be noisy; i.e., some positives may actually be negatives. Since this setting is a strict generalization of the previous discussion, we will overload terminology and use $\fprLU$ again as the positive-unlabeled false positive rate. 

In addition to previous difficulties, we now also cannot estimate the true positive rate $\tpr$, because we do not have access to an unbiased sample from $\cone$; rather, we only have access to a sample contaminated with negatives. Nonetheless, we can express all of the desired rates in terms of only rates for the non-traditional classifier.
\begin{theorem}\label{thm_main}
For a given classifier $\cls:\xset \rightarrow \yset$, the true positive rate $\tpr$ and the false positive rate $\fpr$ can be expressed in terms of the positive-unlabeled $\tprLU$ and $\fprLU$
\begin{align}
{\tpr}&= \frac{(1-{\mpu})\tprLU-(1-{\mpl})\fprLU}{{\mpl}-{\mpu}} \label{eq:tprLU}\\
{\fpr}&= \frac{{\mpl}\fprLU - {\mpu}\tprLU}{{\mpl}-{\mpu}} \label{eq:fprLU}
.
%    {\fpr} &= \frac{\fprPU - \frac{\mpu}{\mpl}{\tprLU}}{1- 2 \mpu + \mpu \mpl} \label{eq:fprLU}\\
 %   \tpr &= \frac{\tprLU - (1-\mpl)  \fpr}{\mpl} \label{eq:tprLU}
\end{align}
The precision $\pr$ can either be converted from a positive-unlabeled precision $\prLU$, with $c=\nicefrac{|\Xone|}{(|\Xu|+|\Xone|)}$, as
\begin{equation*}
 \pr= \frac{\mpu(1-\mpu)}{\mpl-\mpu}\paren*{\frac{1-c}{c}\paren*{\frac{\prLU}{1-\prLU}}-\frac{1-\mpl}{1-\mpu}}
 \end{equation*}
or computed directly as
\begin{equation}
\label{eq:pr1}
\pr = \frac{\mpu \tpr}{\fprLU}
.
\end{equation}
Further, consider a family of classifiers $\mc{F}=\braces*{\cls_\fpr}$ indexed by $\fpr \in [0,1]$ where $\fpr$ is the false positive rate of $\cls_\fpr$. Then for the ROC curve obtained from varying $\eta$, 
the $\aucPN$ can be expressed in terms of the positive-unlabeled $\aucPU$ as
\begin{equation}
\text{\aucPN} = \frac{\text{\aucLU} -\frac{1-(\mpl-\mpu)}{2}}{\mpl-\mpu} \label{eq_auc}
.
\end{equation}
Moreover, $\text{\aucPN} > \text{\aucLU}$, if and only if $\text{\aucLU} > \nicefrac{1}{2}$ and $\mpl -\mpu < 1$.
\end{theorem}
\begin{proof}
\begin{align*}
    \fprLU&= \EE_{\mixu} [\cls(x)]\\
        &= \mpu  \EE_{\cone} [\cls(x)] + (1-\mpu) \EE_{\czero} [\cls(x)]\\
        &= \mpu  \tpr + (1-\mpu) \fpr.
\end{align*}
%\begin{equation}
%\begin{aligned}\label{eq:fprLU}
%    \fprLU&= \EE_{\mixu} [\cls(x)]\\
%        &= \mpu  \EE_{\cone} [\cls(x)] + (1-\mpu) \EE_{\czero} [\cls(x)]\\
%        &= \mpu  \tpr + (1-\mpu) \fpr.\\
%\end{aligned}
%\end{equation}
Similarly, we can obtain the true positive rate
\begin{align*}
    \tprLU&= \EE_{\mixl} [\cls(x)]\\
        &= \mpl  \EE_{\cone} [\cls(x)] + (1-\mpl) \EE_{\czero} [\cls(x)]\\
        &= \mpl  \tpr + (1-\mpl) \fpr.
\end{align*}
We can then solve for $\fpr$ and $\tpr$ to get
the result.

Next, we consider the precision. We can directly re-express the precision as
%\begin{align*}
%\pr &= \frac{\mpu \EE_{\cone}[\cls(x)]}{\EE_\mixu[\cls(x)]}\\
%    &=\frac{\mpu \tpr}{\fprLU}.
%\end{align*}

\begin{equation*}
\pr = \frac{\mpu \EE_{\cone}[\cls(x)]}{\EE_\mixu[\cls(x)]} =\frac{\mpu \tpr}{\fprLU}.
\end{equation*}
To obtain a conversion from $\prLU$, first consider
\begin{align*}
%\label{eq:prLU}
\prPU &= \frac{c \EE_{\mixl}[\cls(x)]}{\EE_{c\mixl+(1-c)\mixu}[\cls(x)]}\\
    &=\frac{c \EE_{\mixl}[\cls(x)]}{c \EE_{\mixl}[\cls(x)] + (1-c)\EE_{\mixu}[\cls(x)]}\\
    &=\frac{1}{1  + \frac{1-c}{c}\frac{\EE_{\mixu}[\cls(x)]}{\EE_{\mixl}[\cls(x)]}}
\end{align*}
We can express a component of this as
\begin{align*}
\frac{\EE_{\mixl}[\cls(x)]}{\EE_{\mixu}[\cls(x)]} \! &= \frac{\mpl\EE_{\cone}[\cls(x)]  + (1-\mpl)\EE_{\czero}[\cls(x)]}{\EE_{\mixu}[\cls(x)]}\\
  %&= \resizebox{0.8\hsize}{!}{$\frac{\mpl}{\mpu}\paren*{\frac{\mpu\EE_{\cone}[\cls(x)]}{\EE_{\mixu}[\cls(x)]}}  + \frac{1-\mpl}{1-\mpu} \paren*{\frac{(1-\mpu)\EE_{\czero}[\cls(x)]}{\EE_{\mixu}[\cls(x)]}}$}\\
  &= \frac{\mpl}{\mpu}\frac{\mpu\EE_{\cone}[\cls(x)]}{\EE_{\mixu}[\cls(x)]}  + \frac{1-\mpl}{1-\mpu} \frac{(1-\mpu)\EE_{\czero}[\cls(x)]}{\EE_{\mixu}[\cls(x)]}\\
    &=\frac{\mpl}{\mpu}\pr  + \frac{1-\mpl}{1-\mpu} (1-\pr)\\
    &=\frac{\mpl-\mpu}{\mpu(1-\mpu)}\pr + \frac{1-\mpl}{1-\mpu}
\end{align*}
where rearranging gives the result.

Next, we derive an equation that allows estimation of the \aucPN\ directly from the \aucLU, $\mpu$ and $\mpl$. Consider a family of classifiers $\mc{F}=\braces*{\cls_\fpr}$ indexed by $\fpr \in [0,1]$ where $\fpr$ is the false positive rate of $\cls_\fpr$. We can express the $\tpr,\fprLU,\tprLU$ of $\cls_\fpr$ as a function of $\fpr$ as follows:
\begin{align*}
\tpr(\fpr)&= \EE_{\cone} [\cls_\fpr(x)],\\
\fprLU(\fpr)&= \EE_{\mixu} [\cls_\fpr(x)]\\
            &= \mpu \tpr(\fpr) + (1-\mpu) \fpr,\\
\tprLU(\fpr)&= \EE_{\mixl} [\cls_\fpr(x)]\\
            &= \mpl \tpr(\fpr) + (1-\mpl) \fpr.
\end{align*}
\newcommand{\pluspace}{\ \ \ \ }  
By definition, the expression for \aucLU\ is
\begin{align*}
  \!\!\!\!\!\! \aucLU \! &=\int_0^1 \tprLU(\fpr) \frac{d\fprLU(\fpr)}{d\fpr} d\fpr\\
                &= \int_0^1 \paren*{\mpl \tpr(\fpr) + (1-\mpl) \fpr} \paren*{\mpu\frac{d\tpr(\fpr)}{d\fpr}+(1-\mpu)} d\fpr\\
                &= \mpu\mpl\int_0^1 \tpr(\fpr) \frac{d\tpr(\fpr)}{d\fpr} d\fpr + (1-\mpu)\mpl\int_0^1 \tpr(\fpr)d\fpr \\
                &\pluspace+ \mpu(1-\mpl)\int_0^1 \hspace{-4pt}\fpr \frac{d\tpr(\fpr)}{d\fpr} d\fpr+ (1-\mpu)(1-\mpl)\int_0^1 \hspace{-4pt} \fpr d\fpr\\
\end{align*}
Now solving for each integral, we obtain
\begin{align*}
    \text{\aucLU}
                &=\frac{\mpu\mpl}{2} [\tpr^2(1)-\tpr^2(0)] +(1-\mpu)\mpl\text{\aucPN} \\
                &\pluspace+ \mpu(1-\mpl)\brac*{\brac*{\fpr\tpr(\fpr)}_0^1 - \int_0^1 \tpr(\fpr) d\fpr}\\
                &\pluspace+ \frac{(1-\mpu)(1-\mpl)}{2}[1^2 -0^2]\\
                 &=\frac{\mpu\mpl+2\mpu(1-\mpl)+(1-\mpu)(1-\mpl)}{2} \\
                 &\pluspace+\brac*{(1-\mpu)\mpl -\mpu(1-\mpl)} \text{\aucPN}\\
                &=\frac{1-(\mpl-\mpu)}{2} + (\mpl-\mpu)\text{\aucPN}
\end{align*}
Rearranging the terms gives the desired result. Finally, from Equation \ref{eq_auc}, we see that
\begin{equation*}
    \text{\aucPN} - \text{\aucLU} = \frac{1-(\mpl-\mpu)}{\mpl-\mpu}\paren*{\text{\aucLU} -\frac{1}{2}}\\
\end{equation*}
proving $\text{\aucPN} > \text{\aucLU}$, if and only if $\text{\aucLU} > \nicefrac{1}{2}$ and $\mpl -\mpu < 1$.
% $$\text{\aucPN} = \frac{\text{\aucLU} -\frac{1-(\mpl-\mpu)}{2}}{\mpl-\mpu}.$$
\end{proof}

\section{Experiments and results}

\subsection{Data sets and classification models}
Our estimators were evaluated using twelve real-life data sets from the UCI Machine Learning Repository \cite{Lichman2013}. All data sets were appropriately modified for binary classification; e.g., regression problems were converted into classification problems based on the mean of the target variable, whereas multiclass classification problems were converted into binary problems by combining classes. When needed, categorical features were converted into numerical features based on the sparse binary representation. 

Classifiers were constructed as ensembles of 100 feed-forward neural networks \cite{Breiman1996}. Each network had five hidden neurons and was trained using resilient propagation \cite{Riedmiller1993}. A validation set containing 25\% of the training data was used to terminate training. For simplicity, no training parameters were varied. Accuracies were estimated using the out-of-bag approach.

\subsection{Experimental protocols}
\label{sec:EP}
To evaluate the quality of performance estimation we first established the ground truth performance of a model by estimating accuracy in a standard supervised setting. All positive examples in all data sets were considered positive and all negative examples were considered negative. A model was then constructed and evaluated for its performance. 

We next simulated the positive-unlabeled setting where we randomly included 1,000 examples (or 100 for smaller data sets) in the positive data set $\Xl$. The number of actual positive examples in each labeled set was a function of parameter $\beta \in \{1, 0.95, 0.75 \}$. For example, when $\beta=1$, all positively labeled examples were positive, and when $\beta<1$, an appropriate fraction of the (positively) labeled data set $\Xl$ was filled with negatives. The remaining examples (positive and negative) were declared unlabeled (data set $\Xu$). The size of the unlabeled data was limited to 10,000 (where relevant) and the fraction of positives in the unlabeled data was used as true $\alpha$. Using all positively labeled examples as positives and all unlabeled examples as negatives, we then estimated the performance of the model in the positive-unlabeled setting. All experiments were repeated fifty times by randomly selecting positives and negatives for the labeled data.

We used our methodology from the previous Section to recover the true accuracy of a model. To recover the area under the ROC curve, we used the \emph{direct conversion} (D) from Equation \ref{eq_auc} as well as \emph{indirect conversion} (I) where traditional true positive and false positive rates were recovered using Equations \ref{eq:tprLU}-\ref{eq:fprLU} for every threshold and then used to reconstruct the ROC curve. In the case of recovering the pr-rc curve, only the indirect conversion was used (using Equations \ref{eq:tprLU} and \ref{eq:pr1}) as no direct conversion formula is known to us. The full algorithm for the indirect recovery is given in the arXiv supplement of this paper.

All experiments were carried out ($i$) by assuming that the class prior $\alpha$ and noise fraction $\beta$ were known (R), and ($ii$) by estimating $\alpha$ and $\beta$ from positive and unlabeled data (E). These experiments were carried out to quantify the performance loss due to the inability to perfectly estimate $(\alpha ,\beta)$. Class priors and noise fraction were estimated using the AlphaMax algorithm \cite{Jain2016,jain2016estimating}. Several recent studies have determined good performance of AlphaMax (Jain et al.~\citeyear{Jain2016}; Jain, White, and Radivojac~\citeyear{jain2016estimating}; Ramaswamy et al.~\citeyear{ramaswamy2016mixture}), in both clean and noisy setting.

The direct recovery methods using real and estimated $(\alpha ,\beta)$ are hereafter referred to as DR and DE methods, respectively, whereas the indirect recovery methods are similarly referred to as IR and IE methods. All four approaches were used to evaluate the estimated AUCs and only IR and IE methods were used to evaluate the estimated area under the pr-rc curve (AUC-PR). 

\subsection{Results}

\begin{figure*}
%\begin{picture}(0,0)
\includegraphics[width=\textwidth]{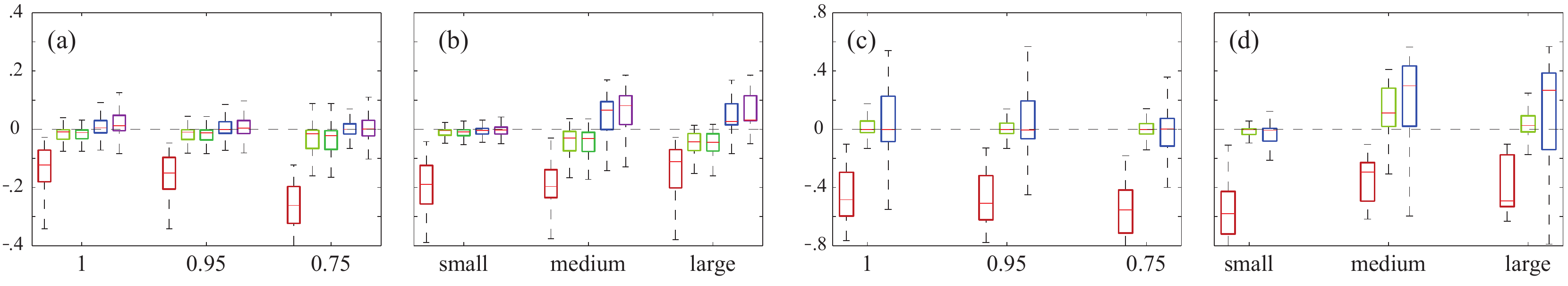}
%\subfloat[B]{\includegraphics[width=0.5\textwidth]{figures/AUCbeta.eps}}
\vspace{-0.3cm}
\caption{The distribution of error of AUC (a, b) and AUC-PR (c, d) estimators on the data generated from the 12 datasets.
%, as described in subsection \ref{sec:EP}. 
\textcolor[rgb]{1,0,0}{PU} represents the estimates on the positive unlabeled data without correction.  \textcolor[rgb]{0.75,0.75,0}{IR}, \textcolor[rgb]{0,0.5,0}{DR}, \textcolor[rgb]{0,0,1.0}{IE}, \textcolor[rgb]{0.75,0,0.75}{DE} are the corrected estimates,
either using the {\bf R}eal values of $(\mpu,\mpl)$ or the {\bf E}stimated values. \textbf{D} indicates that the $\aucLU$ was {\bf D}irectly corrected using equation \eqref{eq_auc} (direct conversion is not done for AUC-PR) and \textbf{I} indicates {\textbf I}ndirect correction by first correcting for the ROC or pr-rc curves. AUC estimates above $1$ were clipped. The x-axis for the left column is the real value of $\mpl$ (in increasing noise order) and for the right column it is the absolute error of $\mpl -\mpu$ estimate binned into small: $[0,0.1)$, medium: $[0.1,0.2)$ and large: $[0.2,\infty)$.}
\end{figure*}

Figure 2 shows the general trends in estimating AUC and AUC-PR over all data sets. Detailed dataset-specific evaluations over all summary statistics are given in Tables 1-2, while the error between the true and recovered performance is further characterized in Figures 3-4. Tables 1-2 and Figures 3-4 are shown in the arXiv supplement of this paper.
%(Tables \ref{table_auc}, \ref{table_pr})

Figure 2(a) shows that, as expected, $\aucLU$ consistently underestimates the true performance. Moreover, it deteriorates with increase in noise. On the other hand, using the correct values for $\alpha$ and $\beta$ (IR and DR, corresponding to the yellow and green boxes) leads to excellent performance over all values of $\beta$. Replacing the true $(\alpha,\beta)$ by their estimates obtained from AlphaMax did not lead to significantly different performance estimates (IE and DE, corresponding to the blue and purple boxes). Since class prior estimation guarantees identifiability of only the upper bounds of $(\alpha,\beta)$, the observed differences are reasonable. Although the aggregate performance of direct and indirect estimation is similar, a detailed comparison between these methods (DR vs. IR and DE vs. IE) provides evidence that the indirect method was superior in both cases ($P = 6.5 \cdot 10^{-6}$ for real $\alpha$ and $\beta$ and $P = 5.7 \cdot 10^{-3}$ for estimated $\alpha$ and $\beta$; one-sided binomial test). Full details of these comparisons are shown in the arXiv supplement. 

Figure 2(b) shows that the performance breaks down with increase in the absolute error of estimates of $\beta - \alpha$. We selected this criterion because the term $\beta - \alpha$ appears in the denominator of Equation \ref{eq_auc} and thus could significantly influence the quality of performance. The increase in error more strongly affects the estimators with approximate $(\alpha,\beta$). Interestingly, the estimators IR and DR both underestimate, and IE and DE both overestimate. We note that in some cases the data sets obtained from UCI Machine Learning Repository may not be perfectly labeled in the first place. This suggests that our ground truth performance might be slightly biased for some data sets which would lead to a situation that the estimated performance is in fact more accurate than observed.

%PRECISION-RECALL
Figures 2(c) and 2(d) show the equivalent plots for AUC-PR from which similar conclusions can be drawn. However, errors in the uncorrected AUC-PR estimates (red boxes) are much higher in comparison. Estimating AUC-PR is therefore not particularly meaningful in the non-traditional setting because precision is sensitive to the proportion of labeled positives in the data set; i.e., $\nicefrac{|\Xone|}{(|\Xu|+|\Xone|)}$, whereas $\tpr$ and $\fpr$ are not.

\section{Related work}

\subsection{Evaluation metrics}
Two-dimensional performance characterization such as ROC or pr-rc curves and the summary statistics based on them have become mainstream in empirical evaluation of classification performance \cite{Flach2003,Fawcett2006,Davis2006,Boyd2012,Clark2013,Flach2015}. Of particular interest to our work is the well-explored relationship between these performance metrics and class priors. For example, Hern\'{a}ndez-Orallo et al.~\shortcite{Hernandez2012} use class priors and area under the ROC curve to compute the expected classification accuracy, whereas Boyd et al.~\shortcite{Boyd2012} relate class priors to the size of the unachievable region in pr-rc space. In the domain of positive-unlabeled learning, Menon et al.~\shortcite{Menon2015} give the relationship between traditionally and non-traditionally evaluated balanced error rates and AUCs of a given classifier. They use this relationship to demonstrate that constructing a non-traditional classifier by optimizing non-traditional AUC results in an optimal traditional AUC. Claesen et al.~\shortcite{Claesen2015} similarly argue the importance of class priors and show how to compute bounds on the true ROC or pr-rc curves. In contrast, our approach directly estimates the unknown statistics and derives a closed-form conversion formula for recovering the area under the ROC curve from the first principles. Another similar work, although in the area of structured-output learning, is by Jiang et al.~\shortcite{Jiang2014} who studied the impact of sequential completion of the (structured) target variable; however, their work makes fewer assumptions on the data distributions and does not lead to the recovery of true performance.

\subsection{Class prior and noise estimation}
Though class prior ($\mpu$) estimation in positive-unlabeled learning is nontrivial, several algorithms have recently emerged in the literature. Elkan and Noto~\shortcite{Elkan2008} estimate the priors from the probability obtained by calibrating the scores of a non-traditional classifier under strong assumptions that the class-conditional distributions do not overlap. The same assumptions are used by \cite{duPlessis2014b} who estimate the class prior as the minimizer of the Pearson divergence. du Plessis et al.~\shortcite{du2015class} improve the method by using penalized \emph{f}-divergence to allow overlap. Blanchard et al.~\shortcite{Blanchard2010} and Jain et al.~\shortcite{Jain2016} showed that class prior estimation, in general, is an ill-posed problem and introduce an ``irreducibility`` constraint on the distribution of the negatives that makes the problem well defined. Blanchard et al.~\shortcite{Blanchard2010} estimate the class prior as the slope of the right endpoint of the empirical ROC curve from non-traditional classifiers while Sanderson and Scott~\shortcite{sanderson2014class} use a fitted curve instead of the actual ROC curve to smooth large noise at endpoints. Loosely speaking, the ROC approach is based on the fact that the class prior under the irreducibility assumption is the minimum value attained by the ratio of the unlabeled and positive sample densities \cite{Jain2016}. Jain et al.~\shortcite{Jain2016} also give an algorithm, \AlgName, a nonparametric maximum likelihood based approach suitable for high-dimensional data. Ramaswamy et al.~\shortcite{ramaswamy2016mixture} give an algorithm based on embedding distributions into a reproducing kernel Hilbert spaces. 

In the case of noisy positives, Scott et al.~\shortcite{Scott2013} and Jain, White, and Radivojac~\shortcite{jain2016estimating} impose a ``mutual irreducibility'' constraint on the distribution of positives and negatives, to make the class prior and the noise proportion estimation well defined. Jain, White, and Radivojac~\shortcite{jain2016estimating} estimate $\mpu,\mpl$ by combining the outputs of two executions of \AlgName, one of which flips the role of positive and unlabeled samples.   

\section{Conclusions}
In this paper we propose simple methods for correcting the estimated performance of classifiers trained in the positive-unlabeled setting. We prove a fundamental result about the relationship between widely-used performance measures and their positive-unlabeled counterparts. The resulting estimators were evaluated over a diverse group of data sets to show that it is feasible and practical to obtain accurate estimates of a classifier's performance in the task of discriminating positive and negative examples. 

The corrected performance measures were uniformly more accurate than the positive-unlabeled estimates, which typically underestimated the performance. Furthermore, we showed that the indirect method for performance recovery outperformed the direct method. This notwithstanding, we do not recommend stopping the established practice of reporting $\textrm{perf}^{\, pu}$; rather we propose that the corrected performance measures should also be provided. In domains where $\alpha$ and $\beta$ are unknown, such estimates will contribute to a better understanding of a classifier's performance and a deeper understanding of the domain itself.

\section{Acknowledgements}
We thank Prof.~Michael W.~Trosset, Kymerleigh A. Pagel and Vikas Pejaver for helpful comments. Grant support: NSF DBI-1458477, NIH R01MH105524, NIH R01GM103725, and the Indiana University Precision Health Initiative.

{\small
\bibliography{bibliography}
}
\bibliographystyle{aaai}

\newpage
\appendix
\section{Appendix}
This appendix describes the indirect method for recovering the area under the ROC curve (AUC) and the area under the precision-recall curve (AUC-PR). Additional characterization of the quality of recovered AUCs is then provided over the entire range of estimated $\alpha$ and $\beta$. Finally, full dataset-specific results are summarized, including statistical tests for comparing direct and indirect recovery methods.

\subsection{Indirect estimators}
%In this section, we provide methods to estimate each of the measures, using the results in Theorem \ref{thm_main}.

%In absence of uncontaminated positive and negative samples $\tpr$ and $\fpr$ cannot be estimated directly. However, an indirect estimation is possible by estimating $\tprLU$ and $\fprLU$ instead. Let $\Xu=[\xu_i]$ and $\Xl=[\xl_i]$ be the unlabeled and the noisy-labeled samples, respectively. 
As mentioned before, we estimate $\tprLU$ and $\fprLU$ as empirical means of $\cls(x)$ over $\Xl$ and $\Xu$, respectively; i.e.,
\begin{align*}
\tprLUhat&=\frac{1}{|\Xl|}\sum_{x \in \Xl}\cls(\xl),\\    
\fprLUhat&=\frac{1}{|\Xu|}\sum_{x \in \Xu}\cls(\xu).
\end{align*}
We then estimate (recover) $\fpr$ and $\tpr$ by replacing $\mpu$, $\mpl$, $\fprLU$ and $\tprLU$ with their estimates in Equations \ref{eq:tprLU} and \ref{eq:fprLU} as
\begin{equation}
\begin{aligned}\label{eq:fprtprhat}
\hat{\tpr}&= \frac{(1-\hat{\mpu})\tprLUhat-(1-\hat{\mpl})\fprLUhat}{\hat{\mpl}-\hat{\mpu}}\\
\hat{\fpr}&= \frac{\hat{\mpl}\fprLUhat - \hat{\mpu}\tprLUhat}{\hat{\mpl}-\hat{\mpu}}
.
\end{aligned}
\end{equation}
To recover the ROC curve we estimate $\fpr$ and $\tpr$ as in Equation \ref{eq:fprtprhat} across different classifiers. However, one needs to be careful because $\hat{\fpr}$ and $\hat{\tpr}$ can take values inconsistent with theory; i.e., outside of the $[0,1]$ range. For example $\hat{\fpr}$ and $\hat{\tpr}$ can be negative; moreover, there is no guarantee on the monotonicity between $\hat{\fpr}$ and $\hat{\tpr}$. We provide an algorithm to correct $\hat{\tpr}$ versus $\hat{\fpr}$ curve in Algorithm \ref{alg:correction}.

% BEGIN ALGORITHM
% L means puts on left side of page
%\begin{wrapfigure}{L}[0pt]{0.5\textwidth}
%\vspace{-0.5cm}
%\begin{minipage}{0.5\textwidth}
\begin{algorithm}[H]
\caption{Algorithm for recovering the ROC curve} \label{alg:correction}
\begin{algorithmic}[] 
\REQUIRE $\hat{\alpha}$ and $\hat{\beta}$ and vectors $\vec{\fpr}^{pu}, \vec{\tpr}^{pu}$, where $i$\textsuperscript{th} entry contains the estimate of $\fprLU$, $\tprLU$ pair coming from the same classifier $h_i$.
\ENSURE vectors  $\vec{\fpr}, \vec{\tpr}$, where $i$\textsuperscript{th} entry contains the $\fpr$, $\tpr$ estimates for classifier $h_i$. The curve $\vec{\fpr}$ versus $\vec{\tpr}$ satisfies the properties of an ROC curve; i.e.,  monotonicity and restriction to the region between (0,0) and (1,1) 
\STATE // Apply Equation \ref{eq:fprtprhat} to get initial $\fpr$, $\tpr$ estimates.
\begin{equation*}
\vec{\fpr} \leftarrow \frac{\hat{\mpl}\vec{\fpr}^{pu} - \hat{\mpu}\vec{\tpr}^{pu}}{\hat{\mpl}-\hat{\mpu}},
\vec{\tpr}\leftarrow \frac{(1-\hat{\mpu})\vec{\tpr}^{pu}-(1-\hat{\mpl})\vec{\fpr}^{pu}}{\hat{\mpl}-\hat{\mpu}}.
\end{equation*}

\STATE // Remove indices for which $\hat{\fpr}$ or $\hat{\tpr}$ are outside $[0,1]$\\
%$ix \leftarrow \{i: \vec{\fpr}(i) \notin [0,1]\  \text{or} \ \vec{\tpr}(i) \notin [0,1]\}$\\
%$\vec{\fpr} \leftarrow \text{remove}(\vec{\fpr},ix), \vec{\tpr} \leftarrow \text{remove}(\vec{\tpr},ix)$

\STATE // Sort $\vec{\fpr}$ in ascending order and reorder the entries in $\vec{\tpr}$ accordingly.

\STATE // Make the curve non-decreasing by replacing the non-increasing values of $\tprhat$ by the largest value to its left. 
%\FOR{$j = 1, \ldots, n_\alpha$}
%\STATE $a$  
%\ENDFOR
% MARTHAC: this is old right?
%\STATE // Smooth $\ell$ using median of $2k$-nearest neighbors; typically, $k=3$
%\STATE $\ell_{\text{smooth}} \gets \ell$
%\FOR{$j = k+1, \ldots, (n_\alpha - k)$}
%\STATE $\ell_{\text{smooth}}(j) \gets \text{median}(\ell(j-k), \ldots, \ell(j+k))$
%\ENDFOR
%\STATE $\ell \gets \ell_{\text{smooth}}$
%\STATE // Scale $\ell$ between $0$ and $1$
%\STATE $\ell \gets \nicefrac{(\ell - \text{min}(\ell))}{(\text{max}(\ell) -\text{min}(\ell))}$
%\STATE // Compute the difference between slopes before and after $j$ using window $win$.
%
%\STATE // Divide by $1-\ell$ plus a small positive constant $\epsilon$.
%\STATE heuristic$ \gets \nicefrac{\Delta\text{slope}}{(1-\ell + \epsilon)}$.
%\STATE $\alpha \leftarrow  c(\text{index-of-max}(\text{heuristic})).$
\end{algorithmic}
\end{algorithm}
%\end{minipage}
%\vspace{-0.5cm}
%\end{wrapfigure}
  % END ALGORITHM

% 
% Next we derive a correction formula for precision-recall curves. Recall being same as the true positive rate, can be corrected using equation \ref{eq:fprtprhat}. Precision of $\cls$ is given by
%\begin{equation}
%\label{eq:pr1}
%\begin{aligned}
%\pr &= \frac{\mpu \EE_{\cone}[\cls(x)]}{\EE_\mixu[\cls(x)]}\\
%    &=\frac{\mpu \tpr}{\fprLU}.
%\end{aligned}
%\end{equation}
%
%Let $c$ be the proportion of positives in the combined (positive and unlabeled) data, i.e., $c=\nicefrac{|\Xone|}{(|\Xu|+|\Xone|)}$. Treating unlabeled examples as negatives leads to the \LU \ precision
%\begin{equation}
%\label{eq:prLU}
%\begin{aligned}
%\prPU &= \frac{c \EE_{\mixl}[\cls(x)]}{\EE_{c\mixl+(1-c)\mixu}[\cls(x)]}\\
%    &=\frac{c \EE_{\mixl}[\cls(x)]}{c \EE_{\mixl}[\cls(x)] + (1-c)\EE_{\mixu}[\cls(x)]}\\
%    &=\frac{1}{1  + \frac{1-c}{c}\frac{\EE_{\mixu}[\cls(x)]}{\EE_{\mixl}[\cls(x)]}}\\
%\end{aligned}
%\end{equation}
%
%\begin{equation}
%\label{eq:prPU}
%\begin{aligned}
% \frac{\EE_{\mixl}[\cls(x)]}{\EE_{\mixu}[\cls(x)]} &= \frac{\mpl\EE_{\cone}[\cls(x)]  + (1-\mpl)\EE_{\czero}[\cls(x)]}{\EE_{\mixu}[\cls(x)]}\\
%  &= \frac{\mpl}{\mpu}\paren*{\frac{\mpu\EE_{\cone}[\cls(x)]}{\EE_{\mixu}[\cls(x)]}}  + \frac{1-\mpl}{1-\mpu} \paren*{\frac{(1-\mpu)\EE_{\czero}[\cls(x)]}{\EE_{\mixu}[\cls(x)]}}\\
%    &=\frac{\mpl}{\mpu}\pr  + \frac{1-\mpl}{1-\mpu} (1-\pr)\\
%    &=\frac{\mpl-\mpu}{\mpu(1-\mpu)}\pr + \frac{1-\mpl}{1-\mpu}
%    \end{aligned}
%\end{equation}

Finally, we can consider two approaches to estimating (recovering) precision. Equation \ref{eq:pr1} suggests estimating the precision as
\begin{equation*}
\prhat= \frac{\hat{\mpu}\tprhat}{\fprLUhat}.
\end{equation*}
Alternatively, the equation can be expressed in terms of $\prLU$ as follows
$$\prhat= \frac{\hat{\mpu}(1-\hat{\mpu})}{\hat{\mpl}-\hat{\mpu}}\paren*{\frac{1-c}{c}\paren*{\frac{\prLUhat}{1-\prLUhat}}-\frac{1-\hat{\mpl}}{1-\hat{\mpu}}},$$
where 
$$\prLUhat=\frac{|\Xone|\tprhat^{pu}}{|\Xone|\tprhat^{pu}+|\Xu|\fprLUhat}.$$

\noindent These estimates are equivalent. The estimate in terms of $\prLUhat$ can be useful if the positive-unlabeled precision is already computed. In general, however, in the absence of this estimate, 
the more direct computation in Equation \ref{eq:pr1} is more desirable. Computing precision directly leads to an equivalent indirect algorithm for recovering the area under the pr-rc curve (AUC-PR).

\subsection{Visualizing errors of AUC estimates}
Figure \ref{fig:heatmaps} shows the absolute difference between the true and recovered areas under the ROC curve as a function of $\hat{\alpha}$ and $\hat{\beta}$ for several combinations of $(\alpha,\beta,\textrm{AUC})$. Specifically, we first selected the true values $(\alpha,\beta,\textrm{AUC})$, from which we calculated $\aucLU$ using Equation 4. For this value of $\aucLU$, we then varied the estimates $\hat{\alpha}$ and $\hat{\beta}$ in $[0,1]$ and used Equation 4 again to compute the recovered area under the ROC curve, $\aucREC$.

The plots suggest that the feasible region for  $(\hat{\alpha},\hat{\beta})$ pairs contains the upper left-hand triangle and lower right-hand triangle. However, since the values in the lower right region lead to the $\aucREC\leq0.5$, this part of the $(\hat{\alpha},\hat{\beta})$ space is not of interest. The middle region corresponds to the estimated AUC values above 1 or below 0 and, therefore, is referred to as infeasible region (Figure \ref{fig:heatmapsDiag}). In our experiments, the estimated AUCs in this region are simply set to 1 or 0, but also suggest problems in the analysis; e.g., that the assumptions may not hold. The size of the infeasible region varies with the true values of $(\alpha,\beta,\textrm{AUC})$, with generally larger values of AUC leading to larger infeasible regions.

Figure \ref{fig:heatmapsDiag} summarizes all panels from Figure \ref{fig:heatmaps}, where each line of interest is characterized as a function of true values of $\alpha$ and $\beta$. When $\hat{\alpha}=0$ and $\hat{\beta}=1$ (upper left-hand corner), the estimated value of AUC equals $\aucLU$. The AUC estimate is correct whenever $\beta-\alpha$ is accurately estimated; i.e., anywhere on the 45$\textdegree$ line $$\hat{\beta}=\hat{\alpha}+\beta-\alpha.$$ %The estimated AUC will be one, when $$\hat{\beta}=\hat{\alpha}+(\beta-\alpha)(2\textrm{AUC}-1).$$ 
The remaining regions of interest are shown in Figure \ref{fig:heatmapsDiag}.

\subsection{Dataset-specific results}

The full results on individual UCI data sets are provided in Table \ref{table_auc} and Table \ref{table_pr}. 

\renewcommand{\multirowsetup}{\centering}
\newlength{\LL} \settowidth{\LL}{Data}
\setlength{\tabcolsep}{3pt}
\begin{table*}[t]
%\label{tab:auc}
\centering
\caption{\small Mean absolute difference between estimate of area under the ROC curve obtained in a traditional setting, AUC, (from data set with labeled positives and negatives) and the uncorrected (\aucPU) and corrected (IR, DR, IE, DE) estimates from dataset containing only noisy positives and unlabeled examples. The mean absolute differences are reported under PU, IR, DR, IE and DE. Twelve data sets from the UCI Machine Learning Repository are used to construct the positive and unlabeled data sets by sampling. IR, DR, IE, DE either use the {\bf R}eal value of $(\mpu,\mpl)$ or the {\bf E}stimated value for correction. \textbf{D} indicates that the AUC from the PU setting was {\bf D}irectly corrected using Equation \ref{eq_auc} and \textbf{I} indicates {\textbf I}ndirect correction by first correcting for the ROC curves. $e$ is the mean absolute error in $\beta -\alpha$ estimates; $d$ gives the dimensions of the data; $n_1$ and $n$ give the number of positives and the total number of points in the original data set, respectively.}\label{table_auc}
\footnotesize
\tracingtabularx
\begin{tabularx}{1.01\linewidth} 
{|>{\setlength{\hsize}{0.100000\hsize}}X| 
>{\setlength{\hsize}{0.050000\hsize}}X|
>{\setlength{\hsize}{0.050000\hsize}}X|
>{\setlength{\hsize}{0.050000\hsize}}X|
>{\setlength{\hsize}{0.080000\hsize}}X|
>{\setlength{\hsize}{0.080000\hsize}}X|
>{\setlength{\hsize}{0.050000\hsize}}X|
>{\setlength{\hsize}{0.050000\hsize}}X|
>{\setlength{\hsize}{0.070000\hsize}}X|
>{\setlength{\hsize}{0.114000\hsize}}X|
>{\setlength{\hsize}{0.114000\hsize}}X|
>{\setlength{\hsize}{0.114000\hsize}}X|
>{\setlength{\hsize}{0.114000\hsize}}X|
>{\setlength{\hsize}{0.114000\hsize}}X|
}\hline
\multirow{1}{1\LL}{Data} & $\mpu$ & $\mpl$ & $d$ & $n_1$ & $n$ & $e$ & AUC & \aucLU 
& \multicolumn{1}{c|}{PU }
& \multicolumn{1}{c|}{IR }
& \multicolumn{1}{c|}{DR }
& \multicolumn{1}{c|}{IE }
& \multicolumn{1}{c|}{DE }
\\ \hline
 \hline\multirow{3}{1\LL}{Bank}  & 0.095 & 1.000 & 13 & 5188 & 45000 & 0.238 & 0.884& 0.842& 0.042 & 0.007 & 0.007 & 0.088 & 0.115 
 \\ 
 & 0.096 & 0.950 & 13 & 5188 & 45000 & 0.230 & 0.884& 0.819& 0.065 & 0.011 & 0.011 & 0.086 & 0.113 
 \\ 
 & 0.101 & 0.750 & 13 & 5188 & 45000 & 0.167 & 0.884& 0.744& 0.140 & 0.010 & 0.011 & 0.085 & 0.111 
 \\ 
\hline 
\multirow{3}{1\LL}{Concrete}  & 0.419 & 1.000 & 8 & 490 & 1030 & 0.143 & 0.940& 0.685& 0.255 & 0.117 & 0.122 & 0.068 & 0.060 
 \\ 
 & 0.425 & 0.950 & 8 & 490 & 1030 & 0.130 & 0.940& 0.661& 0.278 & 0.129 & 0.132 & 0.078 & 0.065 
 \\ 
 & 0.446 & 0.750 & 8 & 490 & 1030 & 0.145 & 0.938& 0.567& 0.371 & 0.201 & 0.216 & 0.196 & 0.206 
 \\ 
\hline 
\multirow{3}{1\LL}{Gas}  & 0.342 & 1.000 & 127 & 2565 & 5574 & 0.006 & 1.000& 0.824& 0.175 & 0.018 & 0.007 & 0.010 & 0.007 
 \\ 
 & 0.353 & 0.950 & 127 & 2565 & 5574 & 0.013 & 1.000& 0.795& 0.205 & 0.003 & 0.007 & 0.012 & 0.016 
 \\ 
 & 0.397 & 0.750 & 127 & 2565 & 5574 & 0.010 & 1.000& 0.672& 0.328 & 0.007 & 0.015 & 0.011 & 0.007 
 \\ 
\hline 
\multirow{3}{1\LL}{Housing}  & 0.268 & 1.000 & 13 & 209 & 506 & 0.063 & 0.950& 0.809& 0.142 & 0.028 & 0.029 & 0.038 & 0.038 
 \\ 
 & 0.281 & 0.950 & 13 & 209 & 506 & 0.055 & 0.951& 0.776& 0.175 & 0.037 & 0.041 & 0.043 & 0.042 
 \\ 
 & 0.330 & 0.750 & 13 & 209 & 506 & 0.079 & 0.951& 0.651& 0.301 & 0.083 & 0.094 & 0.094 & 0.101 
 \\ 
\hline 
\multirow{3}{1\LL}{Landsat}  & 0.093 & 1.000 & 36 & 1508 & 6435 & 0.035 & 0.981& 0.933& 0.048 & 0.004 & 0.004 & 0.005 & 0.015 
 \\ 
 & 0.103 & 0.950 & 36 & 1508 & 6435 & 0.022 & 0.981& 0.904& 0.077 & 0.005 & 0.005 & 0.004 & 0.009 
 \\ 
 & 0.139 & 0.750 & 36 & 1508 & 6435 & 0.020 & 0.981& 0.788& 0.192 & 0.008 & 0.009 & 0.004 & 0.008 
 \\ 
\hline 
\multirow{3}{1\LL}{Mushroom}  & 0.409 & 1.000 & 126 & 3916 & 8124 & 0.006 & 1.000& 0.792& 0.208 & 0.006 & 0.007 & 0.013 & 0.011 
 \\ 
 & 0.416 & 0.950 & 126 & 3916 & 8124 & 0.010 & 1.000& 0.766& 0.234 & 0.003 & 0.005 & 0.009 & 0.012 
 \\ 
 & 0.444 & 0.750 & 126 & 3916 & 8124 & 0.010 & 1.000& 0.648& 0.352 & 0.011 & 0.018 & 0.016 & 0.020 
 \\ 
\hline 
\multirow{3}{1\LL}{Pageblock}  & 0.086 & 1.000 & 10 & 560 & 5473 & 0.413 & 0.970& 0.884& 0.086 & 0.049 & 0.050 & 0.020 & 0.030 
 \\ 
 & 0.087 & 0.950 & 10 & 560 & 5473 & 0.408 & 0.970& 0.858& 0.112 & 0.055 & 0.056 & 0.019 & 0.030 
 \\ 
 & 0.090 & 0.750 & 10 & 560 & 5473 & 0.361 & 0.969& 0.767& 0.202 & 0.059 & 0.064 & 0.018 & 0.031 
 \\ 
\hline 
\multirow{3}{1\LL}{Pendigit}  & 0.243 & 1.000 & 16 & 3430 & 10992 & 0.009 & 0.999& 0.875& 0.124 & 0.004 & 0.004 & 0.004 & 0.002 
 \\ 
 & 0.248 & 0.950 & 16 & 3430 & 10992 & 0.007 & 0.999& 0.847& 0.152 & 0.005 & 0.005 & 0.007 & 0.008 
 \\ 
 & 0.268 & 0.750 & 16 & 3430 & 10992 & 0.010 & 0.999& 0.738& 0.262 & 0.006 & 0.008 & 0.002 & 0.002 
 \\ 
\hline 
\multirow{3}{1\LL}{Pima}  & 0.251 & 1.000 & 8 & 268 & 768 & 0.191 & 0.835& 0.734& 0.101 & 0.026 & 0.028 & 0.070 & 0.090 
 \\ 
 & 0.259 & 0.950 & 8 & 268 & 768 & 0.155 & 0.838& 0.710& 0.128 & 0.038 & 0.040 & 0.060 & 0.069 
 \\ 
 & 0.289 & 0.750 & 8 & 268 & 768 & 0.149 & 0.836& 0.623& 0.213 & 0.070 & 0.075 & 0.064 & 0.073 
 \\ 
\hline 
\multirow{3}{1\LL}{Shuttle}  & 0.139 & 1.000 & 9 & 8903 & 58000 & 0.007 & 1.000& 0.929& 0.071 & 0.001 & 0.002 & 0.015 & 0.005 
 \\ 
 & 0.140 & 0.950 & 9 & 8903 & 58000 & 0.026 & 0.999& 0.903& 0.096 & 0.001 & 0.002 & 0.016 & 0.017 
 \\ 
 & 0.143 & 0.750 & 9 & 8903 & 58000 & 0.004 & 0.999& 0.802& 0.198 & 0.001 & 0.004 & 0.002 & 0.004 
 \\ 
\hline 
\multirow{3}{1\LL}{Spambase}  & 0.226 & 1.000 & 57 & 1813 & 4601 & 0.061 & 0.961& 0.842& 0.118 & 0.018 & 0.018 & 0.013 & 0.020 
 \\ 
 & 0.240 & 0.950 & 57 & 1813 & 4601 & 0.050 & 0.959& 0.812& 0.147 & 0.019 & 0.020 & 0.010 & 0.015 
 \\ 
 & 0.295 & 0.750 & 57 & 1813 & 4601 & 0.057 & 0.961& 0.695& 0.265 & 0.031 & 0.032 & 0.021 & 0.028 
 \\ 
\hline 
\multirow{3}{1\LL}{Wine}  & 0.566 & 1.000 & 11 & 4113 & 6497 & 0.133 & 0.815& 0.626& 0.188 & 0.027 & 0.028 & 0.099 & 0.109 
 \\ 
 & 0.575 & 0.950 & 11 & 4113 & 6497 & 0.121 & 0.816& 0.610& 0.207 & 0.024 & 0.026 & 0.104 & 0.117 
 \\ 
 & 0.612 & 0.750 & 11 & 4113 & 6497 & 0.186 & 0.816& 0.531& 0.285 & 0.095 & 0.104 & 0.158 & 0.158 
 \\ 
\hline 
\end{tabularx}
\normalsize
\end{table*}

\begin{table*}[t]
%\label{tab:aucpr}
\centering
\caption{\small Mean absolute difference between estimate of area under the pr-rc curve obtained in a traditional setting, AUC-PR, (from dataset with labeled positives and negatives) and the uncorrected (AUC-PR$^{pu}$) and corrected (IR, IE) estimates from dataset containing only noisy positives and unlabeled examples. The mean absolute differences are reported under PU, IR and IE. Twelve data sets from the UCI Machine Learning Repository are used to construct the positive and unlabeled datasets by sampling. IR, IE either use the {\bf R}eal value of $(\mpu,\mpl)$ or the {\bf E}stimated value for correction. \textbf{I} indicates {\textbf I}ndirect correction by first correcting for the pr-rc curves. $e$ is the mean absolute error in $\beta -\alpha$ estimates; $d$ gives the dimensions of the data $n_1$ and $n$ give the number of positives and the total number of points in the original data set, respectively.}\label{table_pr}
\footnotesize
\tracingtabularx
\begin{tabularx}{1.01\linewidth} 
{|>{\setlength{\hsize}{0.100000\hsize}}X| 
>{\setlength{\hsize}{0.050000\hsize}}X|
>{\setlength{\hsize}{0.050000\hsize}}X|
>{\setlength{\hsize}{0.050000\hsize}}X|
>{\setlength{\hsize}{0.080000\hsize}}X|
>{\setlength{\hsize}{0.080000\hsize}}X|
>{\setlength{\hsize}{0.050000\hsize}}X|
>{\setlength{\hsize}{0.12000\hsize}}X|
>{\setlength{\hsize}{0.12000\hsize}}X|
>{\setlength{\hsize}{0.116667\hsize}}X|
>{\setlength{\hsize}{0.116667\hsize}}X|
>{\setlength{\hsize}{0.116667\hsize}}X|
}\hline
\multirow{1}{1\LL}{Data} & $\mpu$ & $\mpl$ & $d$ & $n_1$ & $n$ & $e$ & AUC-PR & AUC-PR$^{pu}$ 
& \multicolumn{1}{c|}{ PU }
& \multicolumn{1}{c|}{ IR }
& \multicolumn{1}{c|}{ IE }
\\ \hline
 \hline\multirow{3}{1\LL}{Bank}  & 0.095 & 1.000 & 13 & 5188 & 45000 & 0.238 & 0.478& 0.319& 0.158 & 0.029 & 0.382 
 \\ 
 & 0.096 & 0.950 & 13 & 5188 & 45000 & 0.230 & 0.482& 0.299& 0.184 & 0.027 & 0.368 
 \\ 
 & 0.101 & 0.750 & 13 & 5188 & 45000 & 0.167 & 0.491& 0.236& 0.255 & 0.030 & 0.358 
 \\ 
\hline 
\multirow{3}{1\LL}{Concrete}  & 0.419 & 1.000 & 8 & 490 & 1030 & 0.143 & 0.914& 0.162& 0.752 & 0.163 & 0.321 
 \\ 
 & 0.425 & 0.950 & 8 & 490 & 1030 & 0.130 & 0.919& 0.158& 0.761 & 0.164 & 0.404 
 \\ 
 & 0.446 & 0.750 & 8 & 490 & 1030 & 0.145 & 0.921& 0.124& 0.796 & 0.214 & 0.552 
 \\ 
\hline 
\multirow{3}{1\LL}{Gas}  & 0.342 & 1.000 & 127 & 2565 & 5574 & 0.006 & 1.000& 0.381& 0.619 & 0.014 & 0.046 
 \\ 
 & 0.353 & 0.950 & 127 & 2565 & 5574 & 0.013 & 1.000& 0.358& 0.642 & 0.014 & 0.026 
 \\ 
 & 0.397 & 0.750 & 127 & 2565 & 5574 & 0.010 & 1.000& 0.270& 0.730 & 0.014 & 0.003 
 \\ 
\hline 
\multirow{3}{1\LL}{Housing}  & 0.268 & 1.000 & 13 & 209 & 506 & 0.063 & 0.905& 0.430& 0.475 & 0.067 & 0.270 
 \\ 
 & 0.281 & 0.950 & 13 & 209 & 506 & 0.055 & 0.909& 0.396& 0.514 & 0.091 & 0.306 
 \\ 
 & 0.330 & 0.750 & 13 & 209 & 506 & 0.079 & 0.924& 0.293& 0.631 & 0.152 & 0.368 
 \\ 
\hline 
\multirow{3}{1\LL}{Landsat}  & 0.093 & 1.000 & 36 & 1508 & 6435 & 0.035 & 0.882& 0.618& 0.265 & 0.041 & 0.033 
 \\ 
 & 0.103 & 0.950 & 36 & 1508 & 6435 & 0.022 & 0.887& 0.569& 0.318 & 0.039 & 0.029 
 \\ 
 & 0.139 & 0.750 & 36 & 1508 & 6435 & 0.020 & 0.911& 0.407& 0.504 & 0.049 & 0.023 
 \\ 
\hline 
\multirow{3}{1\LL}{Mushroom}  & 0.409 & 1.000 & 126 & 3916 & 8124 & 0.006 & 1.000& 0.252& 0.748 & 0.004 & 0.056 
 \\ 
 & 0.416 & 0.950 & 126 & 3916 & 8124 & 0.010 & 1.000& 0.238& 0.762 & 0.003 & 0.019 
 \\ 
 & 0.444 & 0.750 & 126 & 3916 & 8124 & 0.010 & 1.000& 0.179& 0.821 & 0.025 & 0.049 
 \\ 
\hline 
\multirow{3}{1\LL}{Pageblock}  & 0.086 & 1.000 & 10 & 560 & 5473 & 0.413 & 0.840& 0.130& 0.710 & 0.161 & 0.198 
 \\ 
 & 0.087 & 0.950 & 10 & 560 & 5473 & 0.408 & 0.839& 0.119& 0.721 & 0.169 & 0.281 
 \\ 
 & 0.090 & 0.750 & 10 & 560 & 5473 & 0.361 & 0.843& 0.084& 0.759 & 0.178 & 0.340 
 \\ 
\hline 
\multirow{3}{1\LL}{Pendigit}  & 0.243 & 1.000 & 16 & 3430 & 10992 & 0.009 & 0.998& 0.288& 0.710 & 0.010 & 0.025 
 \\ 
 & 0.248 & 0.950 & 16 & 3430 & 10992 & 0.007 & 0.998& 0.265& 0.733 & 0.017 & 0.028 
 \\ 
 & 0.268 & 0.750 & 16 & 3430 & 10992 & 0.010 & 0.998& 0.194& 0.804 & 0.015 & 0.013 
 \\ 
\hline 
\multirow{3}{1\LL}{Pima}  & 0.251 & 1.000 & 8 & 268 & 768 & 0.191 & 0.612& 0.256& 0.356 & 0.070 & 0.224 
 \\ 
 & 0.259 & 0.950 & 8 & 268 & 768 & 0.155 & 0.621& 0.237& 0.383 & 0.085 & 0.228 
 \\ 
 & 0.289 & 0.750 & 8 & 268 & 768 & 0.149 & 0.653& 0.191& 0.462 & 0.106 & 0.254 
 \\ 
\hline 
\multirow{3}{1\LL}{Shuttle}  & 0.139 & 1.000 & 9 & 8903 & 58000 & 0.007 & 0.992& 0.414& 0.578 & 0.009 & 0.192 
 \\ 
 & 0.140 & 0.950 & 9 & 8903 & 58000 & 0.026 & 0.994& 0.386& 0.608 & 0.013 & 0.085 
 \\ 
 & 0.143 & 0.750 & 9 & 8903 & 58000 & 0.004 & 0.994& 0.293& 0.700 & 0.008 & 0.014 
 \\ 
\hline 
\multirow{3}{1\LL}{Spambase}  & 0.226 & 1.000 & 57 & 1813 & 4601 & 0.061 & 0.892& 0.502& 0.390 & 0.054 & 0.060 
 \\ 
 & 0.240 & 0.950 & 57 & 1813 & 4601 & 0.050 & 0.894& 0.468& 0.425 & 0.054 & 0.054 
 \\ 
 & 0.295 & 0.750 & 57 & 1813 & 4601 & 0.057 & 0.917& 0.353& 0.564 & 0.072 & 0.048 
 \\ 
\hline 
\multirow{3}{1\LL}{Wine}  & 0.566 & 1.000 & 11 & 4113 & 6497 & 0.133 & 0.849& 0.209& 0.641 & 0.033 & 0.085 
 \\ 
 & 0.575 & 0.950 & 11 & 4113 & 6497 & 0.121 & 0.854& 0.200& 0.654 & 0.031 & 0.090 
 \\ 
 & 0.612 & 0.750 & 11 & 4113 & 6497 & 0.186 & 0.870& 0.166& 0.703 & 0.086 & 0.441 
 \\ 
\hline 
\end{tabularx}
\normalsize
\end{table*}

%\subsection{Assessing statistical significance}
The AUC comparisons between the direct and indirect method was conducted using the counting tests. Each combination (data set, $\alpha$, $\beta$) was considered to be an independent experiment and the number of wins vs. losses were counted for each algorithm; in case of ties, the wins were distributed in an alternating manner, starting with the direct method, then indirect, and so on. Finally, statistical significance was tested by using a one-sided binomial test where the null hypothesis ($H_{0}$) was that the two algorithms have equal performance and the alternative hypothesis ($H_{1}$) was that the indirect method is more accurate than the direct method. The P-value was calculated as

\[
P=\sum_{i=k}^{n}\binom{n}{i}p^{i}(1-p)^{n-i}
\]

\noindent where $n=36$ is the total number of experiments, $k$ is the number of times the indirect method outperformed the direct method, and $p=\nicefrac{1}{2}$ is the probability of a win for either method under $H_{0}$.

In the case of real values of $\alpha$ and $\beta$ (Table 1, columns DR vs. IR), we observed 1 win for the direct method, 28 wins for the indirect method and 7 ties ($k=31$). This resulted in $P=6.5\cdot10^{-6}$. On the other hand, in the case of the estimated values of $\alpha$ and $\beta$ (Table 1, columns DE vs. IE), we observed 9 wins for the direct method, 25 wins for the indirect method and 2 ties ($k=26$). This resulted in $P=5.7\cdot10^{-3}$.

A possible reason for this outcome may be the sensitivity of the one-step direct conversion from Equation 4 to errors in estimating $\aucLU$ and $\beta-\alpha$, which can frequently land $\aucREC$ in the infeasible region. The indirect method, on the other hand, re-estimates the true positive and false positive rates for each decision threshold to first recover the ROC curve. Although this method seems more sensitive to the errors in estimating $\beta-\alpha$, it allows for removal of problematic points from the ROC curve and, thus, leads to an increased accuracy of estimation. Additional experiments are necessary to further characterize both direct and indirect methods; e.g., the sensitivity of the indirect method to the number of $(\eta, \gamma)$ points used to construct an ROC curve.

\begin{figure*}
\includegraphics{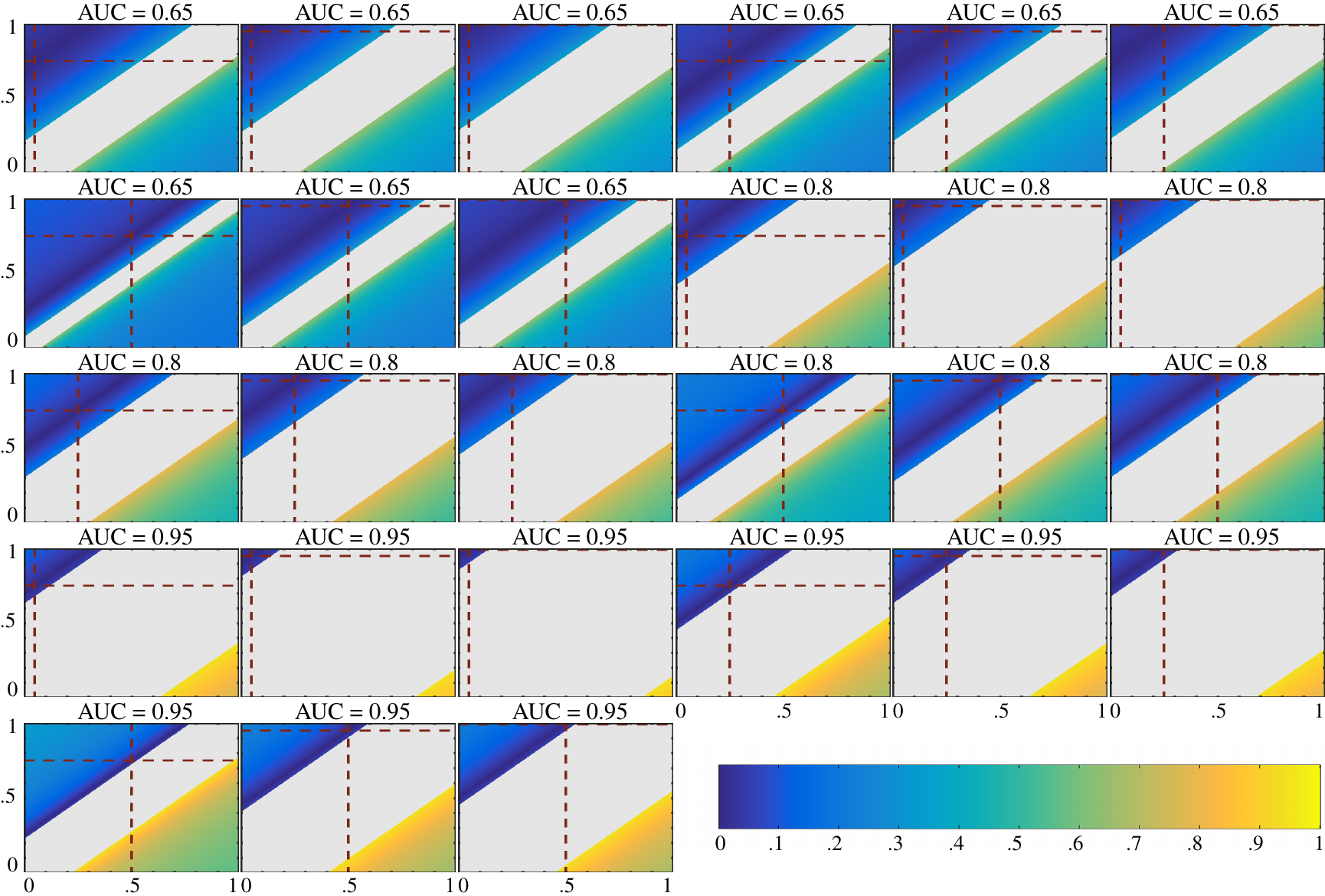}

\vspace{0.25cm}
 \caption{Characterization of the AUC recovery. Each heatmap shows the absolute error between true and recovered AUC values for a different combination of $(\alpha,\beta,\textrm{AUC})$. First, the true $(\alpha,\beta,\textrm{AUC})$ were selected by picking $\alpha\in\left\{ 0.05,0.25,0.50\right\}$, $\beta\in\left\{ 0.75,0.95,1.00\right\}$ and $\textrm{AUC}\in\left\{ 0.65,0.80,0.95\right\}$, from which $\aucLU$ was calculated using Equation 4. Finally, Equation 4 was again used to calculate the recovered AUC, $\aucREC$, for all combinations of estimated $\alpha$ and $\beta$ in the unit square. The colors in the heatmap reflect absolute errors between $\textrm{AUC}$ and $\aucREC$, whereas the gray color around the diagonal indicates the region in which $\aucREC$ is outside of the $[0, 1]$ interval. The x-axis represents the estimated $\alpha$, the y-axis represents the estimated $\beta$, while the true $\alpha$ and $\beta$ are shown by dotted lines. The true AUC is shown on top of each heatmap.}
 \label{fig:heatmaps}
 \end{figure*}

\begin{figure*} 
\centering
 \includegraphics{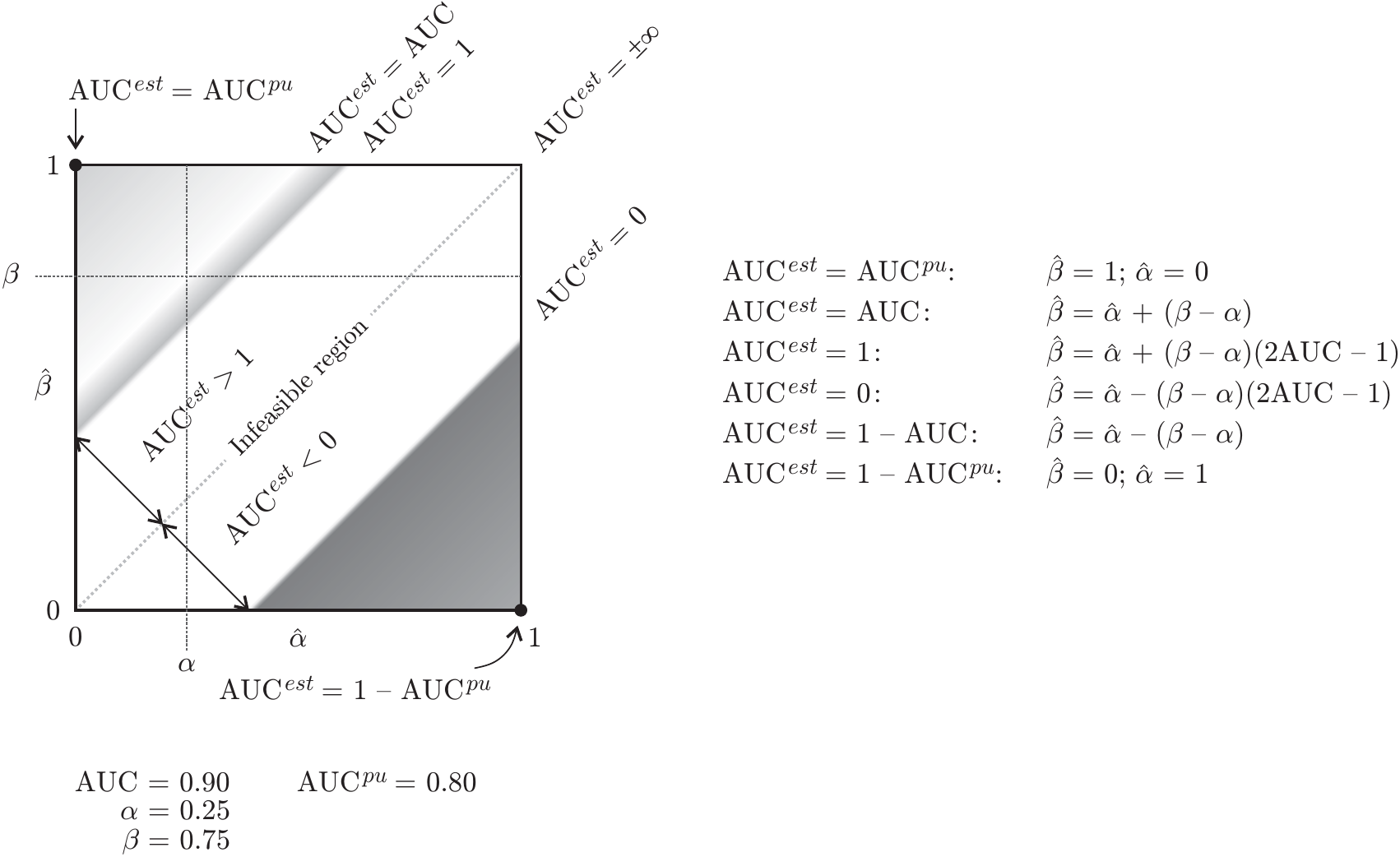}
 \vspace{0.25cm}
 \caption{Characterization of the AUC recovery. The left panel illustrates the absolute error between true and recovered AUC values as a function of estimated $\alpha$ and $\beta$. First, the true $(\alpha,\beta,\textrm{AUC})$ were set to $(0.25, 0.75, 0.90)$, from which the $\aucLU$ was calculated to be $0.80$ from Equation 4. Finally, Equation 4 was again used to calculate the recovered AUC, $\aucREC$, for all estimated $\alpha$ and $\beta$ combinations in the unit square. The lighter shades indicate smaller absolute errors between $\textrm{AUC}$ and $\aucREC$, while the darker shades indicate larger absolute errors, except for the infeasible region that is shown in white. The right panel summarizes notable regions (lines) where the recovered AUC, $\aucREC$, corresponds to particular important values.}
 \label{fig:heatmapsDiag}
\end{figure*} 

\end{document}